\documentclass[11pt,letterpaper]{article}

\usepackage{amsmath,amssymb,amsthm}
\usepackage{mathtools}
\usepackage{hyperref}
\usepackage[margin=1in]{geometry}
\usepackage{booktabs}
\usepackage{enumitem}
\usepackage{graphicx}
\usepackage{xcolor}
\usepackage{lineno}

\newtheorem{theorem}{Theorem}[section]
\newtheorem{lemma}[theorem]{Lemma}
\newtheorem{proposition}[theorem]{Proposition}

\theoremstyle{definition}
\newtheorem{definition}[theorem]{Definition}
\newtheorem{observation}[theorem]{Observation}
\theoremstyle{remark}
\newtheorem{remark}[theorem]{Remark}

\DeclareMathOperator{\rank}{rank}

\DeclareMathOperator{\IPR}{IPR}
\newcommand{\R}{\mathbb{R}}

\title{On the Invariants of Softmax Attention}
\author{Wonsuk Lee\\
Seoul National University\ \& SK Hynix America\\
\texttt{\small wonsuk.lee@snu.ac.kr}}
\date{April 5, 2026}

\begin{document}
\maketitle

\begin{abstract}
Softmax attention maps every query--key interaction into a
probability distribution, but the underlying structure remains
largely unexplored.  We define the \emph{energy field}, the
row-centered attention logit, and show that it exhibits invariant
properties across models, architectures, and inputs.

Two classes of invariants emerge. 
\emph{Mechanism-level} invariants follow from the algebraic structure 
of softmax attention. They include a per-row zero-sum constraint, 
a rank bound determined by the head dimension, and spectral signatures 
that follow from them. 
\emph{Model-level} regularities are not required by the mechanism,
yet hold in every autoregressive language model we test,
spanning several architecture families.  The energy field distributes
its variance over key positions without concentrating at a few.  
This delocalization traces to a property of the key matrix we call 
\emph{key incoherence}.

These invariants have practical consequences.  The rank bound
confines the energy field to a low-dimensional subspace.  Key
incoherence yields a per-head training monitor.  All results are
verified at multiple context lengths and input texts.
\end{abstract}

\section{Introduction}
\label{sec:introduction}

The softmax attention mechanism~\cite{vaswani2017attention} computes
a compatibility score between every pair of query and key vectors,
then normalizes these scores into a probability distribution that
routes information between tokens.  This mechanism is shared by all transformer language models, from
GPT~\cite{radford2019language} to
LLaMA~\cite{touvron2023llama} and
Mistral~\cite{jiang2023mistral}.

Despite its engineering success, why dot-product softmax attention
works so well in transformer language models is not understood.
Theoretical work has studied expressiveness and approximation
power~\cite{vaswani2017attention}, efficient
alternatives~\cite{katharopoulos2020transformers,
choromanski2021rethinking}, and empirical patterns in the attention
weights~\cite{xiao2024efficient}. But what this simple mathematical 
structure produces when applied to language has not been characterized.
Language has no known governing law or theory from which the
effectiveness of attention can be derived.  If structure exists 
in the attention computation, it must emerge from the interaction 
between the softmax mechanism and the data it processes.  
This paper searches for invariants of that interaction, 
as a step toward understanding what softmax attention computes 
beyond what it achieves.

To find these invariants, we define the \emph{energy field} of a
single attention head, the row-centered attention logit.  It
transforms the sparse, bounded attention probability $p_{ij}$
into a representation that reveals what $p_{ij}$ hides
(Figure~\ref{fig:field-gallery}). 
Its properties divide into two categories.

\textbf{Mechanism-level properties} follow from softmax
normalization and the bilinear form
$Z_{ij} = q_i^\mathrm{T} k_j / \sqrt{d_h}$, where $d_h$ is the
head dimension.  The row-sum identity ($\sum_j E_{ij} = 0$) produces an exact
autocovariance bridge in the flattened energy signal. The rank bound ($\le d_h + 1$) confines the energy field to a
low-dimensional subspace.  Twenty SVD components capture 
over 90\% of variance in the energy field at all context
lengths tested.

\textbf{Model-level regularities} emerge from training on
autoregressive language modeling.  They are not required by the
mechanism, yet they hold in every model we test.  The central
finding is \emph{key incoherence}: the key norm concentration
ratio $\mu_K$ stays near 1.5 in all 16 models spanning seven
architecture families.  There is no obvious reason for this
agreement.  The models differ in size, training data, and position
encoding, but the regularity persists.  It implies that the
singular vectors of the energy field delocalize, and that the
field concentrates its variance in far fewer dimensions than the
rank bound allows.

Together, these properties characterize the intrinsic structure
of softmax attention.

We verify all claims across 16 models spanning seven architecture
families (124M--7B parameters), five context lengths ($L = 64$
to $1{,}024$), and five input texts from Project Gutenberg,
totaling 400 measurements.  Every quantity is stable over input texts.

\section{Related Work}
\label{sec:related}

\paragraph{Attention structure and patterns.}
Much prior work on the attention
mechanism~\cite{vaswani2017attention} focuses on the attention
probability $p_{ij}$ rather than the underlying logit structure.
Xiao et al.~\cite{xiao2024efficient} identified ``attention sinks,''
initial tokens that absorb disproportionate attention probability,
and exploited them for efficient streaming inference.  They did not
formalize the structure behind them.  The energy field 
(Definition~\ref{def:energy}) provides that
structure.  Attention sinks and other recurring patterns become measurable
features of the energy field (Section~\ref{sec:field}).

\paragraph{Attention approximations.}
Linear attention~\cite{katharopoulos2020transformers} and the
Performer~\cite{choromanski2021rethinking} replace softmax with
kernel-based approximations to reduce the quadratic cost of
attention.  These approximations break the exact row-sum
cancellation that softmax enforces, a difference visible in the
vanishing DC property (Section~\ref{sec:dwt}).

\paragraph{Matrix incoherence.}
The key incoherence parameter $\mu_K$ measures how uniformly the
key norms spread across positions.  The same quantity appears in
matrix completion~\cite{candes2009exact}, where low-rank recovery
requires the analogous condition $\mu = O(1)$.  Our finding that
trained transformers satisfy this condition connects softmax
attention theory to a well-studied mathematical framework.

\section{The Energy Field}
\label{sec:field}

\subsection{Definitions}

A single attention head~\cite{vaswani2017attention} projects each
token $x_i \in \R^{d_{\text{model}}}$ into a query $q_i = W_Q x_i$ and a key
$k_j = W_K x_j$ in $\R^{d_h}$, where $d_h$ is the head dimension.
The attention logit $Z_{ij} = q_i^\mathrm{T} k_j / \sqrt{d_h}$ gives, in
matrix form, $Z = QK^\mathrm{T} / \sqrt{d_h}$ with
$Q, K \in \R^{L \times d_h}$.  In autoregressive models,
position~$i$ attends only to positions $j \le i$, giving a causal
context of size $n_i = i + 1$.  The attention probability is
$p_{ij} = \mathrm{softmax}_j(Z_{ij}) = e^{Z_{ij}} / \sum_{j' \le i} e^{Z_{ij'}}$,
where $j$ in $\mathrm{softmax}_j$ indicates normalization over key
positions.

The attention probability $p_{ij}$ is sparse, bounded, and entangled 
with a nonlinear partition function. The logit $Z_{ij}$ is unbounded 
and linear but carries a position-dependent baseline. 
A query with large norm produces high dot products with all keys, 
not just the relevant ones. 
This baseline skews the interaction pattern. Row centering removes it.

\begin{definition}[Energy field]\label{def:energy}
The \emph{energy field} is the row-centered logit:
\begin{equation}\label{eq:energy}
E_{ij} = Z_{ij} - \mu_i,
\qquad \mu_i = \frac{1}{n_i}\sum_{j'=0}^{n_i-1} Z_{ij'},
\end{equation}
where $n_i = i + 1$ is the causal context size at position~$i$.
By construction, $\sum_{j=0}^{n_i-1} E_{ij} = 0$ for every row~$i$,
the \emph{row-sum identity}.
\end{definition}

The transformation is invertible.  Softmax is shift-invariant, so
$p_{ij} = \mathrm{softmax}_j(E_{ij})$.  The energy field carries
the information needed to compute the attention probability.

\begin{remark}[Log-partition cancellation]\label{rem:logpartition}
Since $\log p_{ij} = Z_{ij} - \log \mathcal{Z}_i$, the partition
function cancels in mean centering, giving
$E_{ij} = \log p_{ij} - (1/n_i)\sum_{j'}\log p_{ij'}$.
The energy field is both a geometric quantity (centered dot
products in embedding space) and a probabilistic one
(mean-centered log-probability in attention space).
\end{remark}

\begin{definition}[Row-centered logit matrix]\label{def:extended}
The \emph{row-centered logit matrix} $\tilde{E} \in \R^{L \times L}$
is obtained by centering each row of the full logit matrix $Z$ by
its mean over all $L$ key positions:
\begin{equation}\label{eq:extended}
\tilde{E}_{ij} = Z_{ij} - \bar{\mu}_i,
\qquad \bar{\mu}_i = \frac{1}{L}\sum_{j'=0}^{L-1} Z_{ij'}.
\end{equation}
By construction, $\sum_{j=0}^{L-1} \tilde{E}_{ij} = 0$ for every
row~$i$.
\end{definition}

\begin{remark}[Relationship between $E$ and $\tilde{E}$]
\label{rem:E-vs-Etilde}
The causal energy field $E_{ij}$ and the row-centered logit matrix
$\tilde{E}_{ij}$ start from the same logit $Z_{ij}$ but subtract
different row means.  $E_{ij}$ subtracts the causal mean
$\mu_i = (1/n_i)\sum_{j'=0}^{i} Z_{ij'}$, the average over the
$n_i = i+1$ positions the model sees.  $\tilde{E}_{ij}$ subtracts
the full mean $\bar{\mu}_i = (1/L)\sum_{j'=0}^{L-1} Z_{ij'}$, the
average over all $L$ positions.  

The logit $Z_{ij}$ is defined for all $(i,j)$ pairs, including
positions $j > i$ that the causal mask excludes during inference.  
$\tilde{E}_{ij}$ therefore spans the full $L \times L$ matrix, 
while $E_{ij}$ is defined only on the causal triangle ($j \le i$). 
On that triangle, the two differ by a row-dependent constant
$E_{ij} = \tilde{E}_{ij} + (\bar{\mu}_i - \mu_i)$, which shifts
each row uniformly but does not affect within-row correlations or
the rank structure.

The two objects serve different purposes.  The causal energy field
$E_{ij}$ is the \emph{physical} quantity that determines what
softmax produces and drives the spectral signatures through the
causal row-sum $\sum_{j=0}^{i} E_{ij} = 0$.  
The row-centered logit $\tilde{E}$ is the \emph{algebraic} tool,
with clean properties that support the SVD decomposition and
delocalization analysis.
\end{remark}

\subsection{Elementary properties}

Two mechanism-level properties hold for any weights and any input.
The \textbf{row-sum identity} $\sum_j E_{ij} = 0$ ensures that
attention is a zero-sum competition among visible positions.  Its
spectral consequences are developed in
Section~\ref{sec:signatures}.

The \textbf{rank bound} $\rank(\tilde{E}) \le d_h + 1$ follows
because $Z = QK^\mathrm{T} / \sqrt{d_h}$ has rank at most $d_h$, 
and row centering subtracts the rank-one matrix
$\bar{\mu}\,\mathbf{1}^\mathrm{T}$.  By rank subadditivity,
$\rank(\tilde{E}) \le d_h + 1$.

\subsection{A first look}

\begin{figure}[t]
  \centering
  \includegraphics[width=\textwidth]{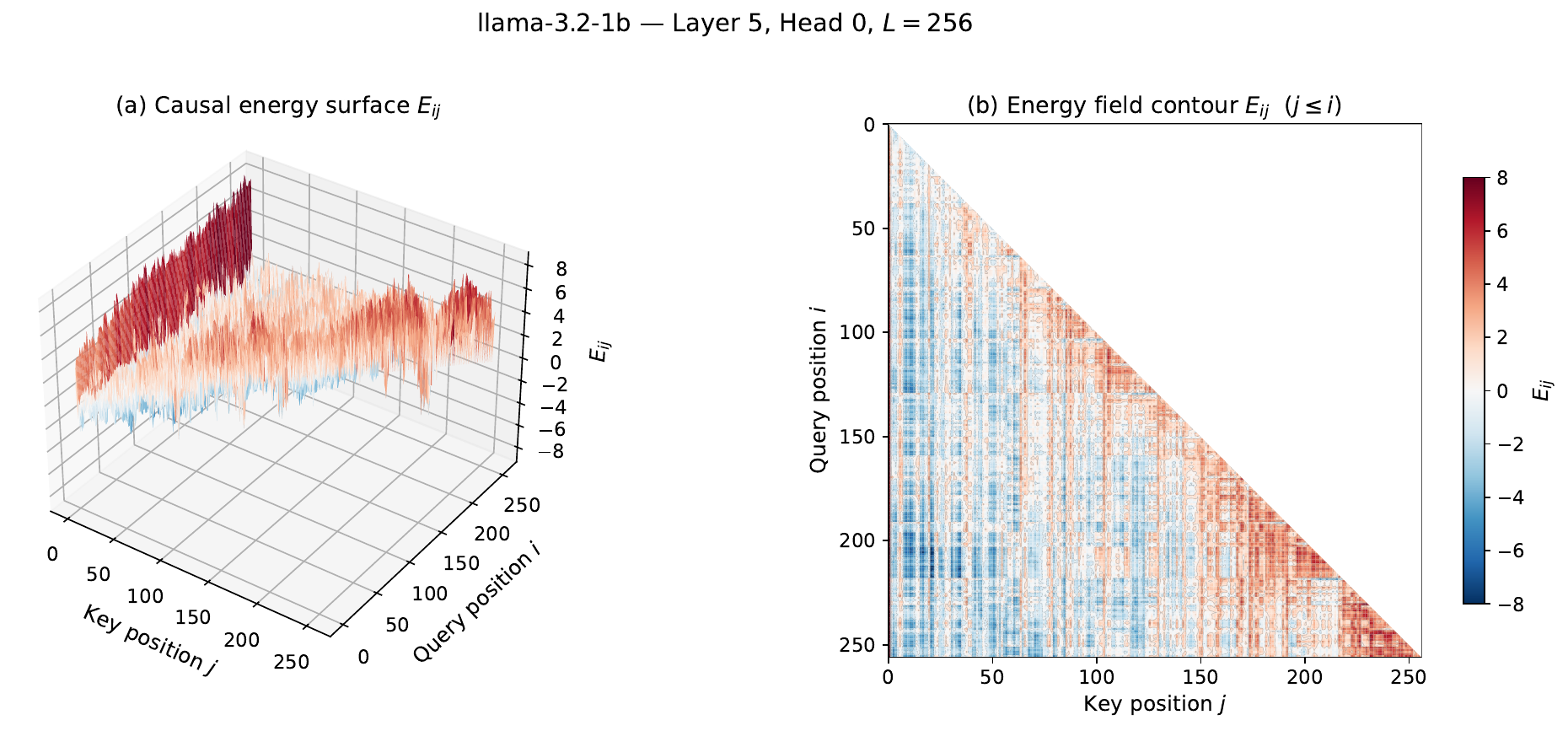}
  \caption{The causal energy field $E_{ij}$ of a single LLaMA-3.2-1B
    head (layer~5, head~0, $L = 256$, processing Dickens).
    \textbf{(a)}~Three-dimensional surface over the causal region
    ($j \le i$).  Ridges (red) mark high energy, valleys
    (blue) mark low energy.  Every row sums to zero, so red and blue
    balance exactly.  The acausal region ($j > i$) is masked.
    \textbf{(b)}~Filled contour plot of the causal energy field,
    with iso-energy lines.  Vertical stripes appear where key
    vectors produce consistently high or low energy for all queries.
    The triangular domain grows with $i$, reflecting the expanding
    causal context.
    Two features dominate.  The diagonal ridge ($j = i$) carries
    mean energy $E_{ii} \approx 1.5$, reflecting self-attention
    since $q_i$ and $k_i$ project from the same embedding.  The red
    band at $j = 0$ is far stronger, with mean
    $E_{i,0} \approx 6.1$.  This is the BOS attention
    sink~\cite{xiao2024efficient}, visible as a column of uniformly
    high energy across all query positions.}

  \label{fig:field-gallery}
\end{figure}

Figure~\ref{fig:field-gallery} shows $E_{ij}$ for a single
LLaMA-3.2-1B head at $L = 256$.  The features are representative
of all models tested
(Sections~\ref{sec:incoherence}--\ref{sec:delocalization}).  Two
patterns dominate: the diagonal ridge ($E_{ii} \approx 1.5$,
self-attention from shared embeddings) and the vertical band at
$j = 0$ ($E_{i,0} \approx 6.1$, the attention
sink~\cite{xiao2024efficient}). Vertical stripes at other positions
reflect key vectors that produce consistently high or low 
energy across all queries (Section~\ref{sec:svd}).

\subsection{The flattened signal}

\begin{definition}[Flattened signal]\label{def:flatten}
We construct the flattened signal by reading the causal energy field
row by row, starting from row~1 (row~0 contains a single zero
entry), to form a one-dimensional signal
\begin{equation}\label{eq:flatten}
\mathcal{E}_1, \mathcal{E}_2, \ldots, \mathcal{E}_N,
\qquad N = \frac{L(L+1)}{2} - 1 \approx \frac{L^2}{2}.
\end{equation}
Each position $t$ maps to a unique row--column pair $(i(t), j(t))$,
with $\mathcal{E}_t = E_{i(t),\, j(t)}$.
\end{definition}

This ordering follows the causal structure of the attention
mechanism.  Consecutive entries within each row are nearby
key positions, and row boundaries mark transitions to new queries.
Since every row sums to zero, the global sum vanishes as
$\sum_{t=1}^{N} \mathcal{E}_t = 0$.  The spectral properties
derived in Section~\ref{sec:signatures} hold for any row-by-row
ordering.

\section{SVD Channel Decomposition}
\label{sec:svd}

The rank bound gives at most $r \le d_h + 1$ nonzero singular
values.  The SVD of $\tilde{E}$,
\begin{equation}\label{eq:svd}
\tilde{E} = \sum_{k=1}^{r} \sigma_k \, u_k \, v_k^\mathrm{T},
\end{equation}
decomposes the energy field into $r$ channels, where $u_k \in \R^L$
are query profiles and $v_k \in \R^L$ are key profiles.

On the causal triangle ($j \le i$), the energy field $E_{ij}$
differs from $\tilde{E}_{ij}$ by the row-dependent constant
$\bar{\mu}_i - \mu_i$ as in Remark~\ref{rem:E-vs-Etilde}, hence
\begin{equation}\label{eq:svd-causal}
E_{ij} = \sum_{k=1}^{r} \sigma_k \, (u_k)_i \, (v_k)_j
  + (\bar{\mu}_i - \mu_i),
\qquad j \le i.
\end{equation}
Each channel $k$ pairs a query profile $u_k$ with a key profile
$v_k$.  Within row~$i$, the energy pattern is a weighted sum of key
profiles with weights $\sigma_k \cdot (u_k)_i$.  The key profiles are
shared across all rows, producing the vertical stripes visible in
Figure~\ref{fig:field-gallery}(b).

\paragraph{Rapid singular value decay.}
The rank bound is loose.  The top 20 singular values of $\tilde{E}$
capture 98\% of total variance at $L = 256$
(Table~\ref{tab:fidelity}), decreasing to 96\% at
$L = 512$--$1{,}024$ as longer sequences contain richer structure
within the subspace.  Trained models concentrate variance in far
fewer components than the mechanism allows.

\paragraph{Channel signals and autocovariance.}
The channel decomposition bounds the correlation structure of the
flattened signal.  This drives the spectral analysis of
Section~\ref{sec:signatures}.  To see how, we define the
\emph{per-channel signal} at position~$t$ using the flattened signal
of Definition~\ref{def:flatten}:
\begin{equation}\label{eq:channel-signal}
s_k(t) \coloneqq \sigma_k \cdot (u_k)_{i(t)} \cdot (v_k)_{j(t)},
\end{equation}
then the full flattened signal is
$\mathcal{E}_t = \sum_{k=1}^{r} s_k(t)$.  The autocovariance of the
flattened signal decomposes into channel cross-covariances:
\begin{equation}\label{eq:channel-cov}
\hat{\gamma}(\tau)
= \frac{1}{N}\sum_{t=1}^{N-\tau}\mathcal{E}_t\,\mathcal{E}_{t+\tau}
= \sum_{k=1}^{r}\sum_{l=1}^{r}\Gamma_{kl}(\tau),
\end{equation}
where
$\Gamma_{kl}(\tau) = (1/N)\sum_{t=1}^{N-\tau} s_k(t)\,s_l(t+\tau)$.
There are at most $r^2 \le (d_h + 1)^2$ such terms.  A signal of
length $N \approx L^2/2$ has its correlation structure captured by a
fixed number of channel interactions determined by head dimension, not
context length. This bounded complexity makes the autocovariance 
decomposition tractable.

\section{Key Incoherence}
\label{sec:incoherence}

The SVD of Section~\ref{sec:svd} decomposes the energy field into
channels, but how are the singular vectors distributed spatially?
Do they spread across all key positions, or concentrate at a few?
The answer depends on the key matrix $K$.

\begin{definition}[Key incoherence]\label{def:mu_K}
For a key matrix $K \in \R^{L \times d_h}$ with rows $k_0, \ldots,
k_{L-1}$, the \emph{key incoherence parameter} is
\begin{equation}\label{eq:mu_K}
\mu_K = L \cdot \frac{\max_j \|k_j\|^2}{\|K\|_F^2},
\end{equation}
where $\|K\|_F^2 = \sum_{j=0}^{L-1} \|k_j\|^2$ is the squared
Frobenius norm.
\end{definition}

This parameter measures how uniformly the key norms are distributed
across positions.  When all keys have equal norm, $\mu_K = 1$.  When
a single key carries all the norm, $\mu_K = L$.  For a random matrix
with i.i.d.\ entries, $\mu_K \to 1$ as $L \to \infty$.

\subsection{Main empirical result}

\begin{observation}[Key incoherence of trained language models]
\label{obs:key_incoherence}
Across 5{,}888 key-value heads in 16 trained language models
(124M--7B parameters), seven architecture families (Pythia, GPT-2,
OPT, Qwen, LLaMA, Phi-2, Mistral), five context lengths
($L = 64$--$1{,}024$), and five input texts from Project Gutenberg:
\begin{equation}\label{eq:mu_K_result}
\mu_K = O(1).
\end{equation}
Mean $\mu_K = 1.5$, maximum $\mu_K = 26$ (a single OPT-350m head). 
Excluding this outlier, the maximum is $8.6$.
The quantity is invariant across text genres
with coefficient of variation below 3\%.
\end{observation}

Layer normalization controls $\|K\|_F^2 / L$, keeping the average
key norm at $O(d_h)$, but does not prevent $\max_j \|k_j\|^2$ from
growing with $L$.  A model that concentrates attention at a fixed
position would have $\mu_K$ growing linearly with $L$.  In
practice, even randomly initialized models produce
$\mu_K \approx 1.5$, because layer normalization equalizes
$\|x_j\|$ over positions and random $W_K$ introduces no positional
bias.  Training develops specialized attention patterns but
preserves this uniformity.  Table~\ref{tab:mu_K} confirms this for
all seven architecture families.

\begin{table}[ht]
\centering
\caption{Key incoherence $\mu_K$ across architecture families
(mean $\pm$ std across all heads at $L = 256$).}
\label{tab:mu_K}
\begin{tabular}{lcrrc}
\addlinespace[1.0em]
\toprule
Family & Models & Heads & $\mu_K$ (mean $\pm$ std) &
  $\mu_K \le 5$ (\%) \\
\midrule
Pythia  & 4 & 1{,}680 & $1.37 \pm 0.4$ & 100.0 \\
GPT-2   & 3 & 1{,}248 & $1.74 \pm 0.5$ & 99.6 \\
OPT     & 2 & 1{,}152 & $1.65 \pm 0.7$ & 99.1 \\
Qwen    & 3 &    176 & $1.48 \pm 0.2$ & 100.0 \\
LLaMA   & 2 &    352 & $1.49 \pm 0.2$ & 100.0 \\
Phi-2   & 1 & 1{,}024 & $1.45 \pm 0.3$ & 99.9 \\
Mistral & 1 &    256 & $1.50 \pm 0.2$ & 100.0 \\
\midrule
\textbf{All} & \textbf{16} & \textbf{5{,}888}
  & $\mathbf{1.5 \pm 0.4}$ & \textbf{99.8} \\
\bottomrule
\end{tabular}
\end{table}

OPT-350m is the sole outlier, with a few heads reaching
$\mu_K = 26$ at $L = 1{,}024$.  OPT was released as a ``warts and
all'' reproduction~\cite{zhang2022opt} with documented training
instabilities.  OPT-1.3B, from the same family, satisfies
$\mu_K < 6$.

\begin{remark}\label{rem:structural}
The row-sum identity and rank bound are mechanism-level, holding for
any weights and any input.  Key incoherence is empirical.  Whether training on non-language
tasks such as vision or protein modeling also preserves it remains
open (Section~\ref{sec:discussion}).
\end{remark}

\subsection{Cross-architecture universality}
\label{sec:universality}

Figure~\ref{fig:mu_K_vs_size} plots $\mu_K$ against model size for
all 16 models, spanning nearly two orders of magnitude in parameter
count (124M to 7B).  The key incoherence parameter shows no trend
with model size, remaining near $\mu_K \approx 1.5$ from GPT-2 
at 124M parameters to Mistral at 7B parameters.  
The interquartile ranges overlap for all architecture families.
Neither the position encoding scheme, the training corpus, 
nor the model depth shifts $\mu_K$ appreciably.

\begin{figure}[t]
  \centering
  \includegraphics[width=0.7\textwidth]{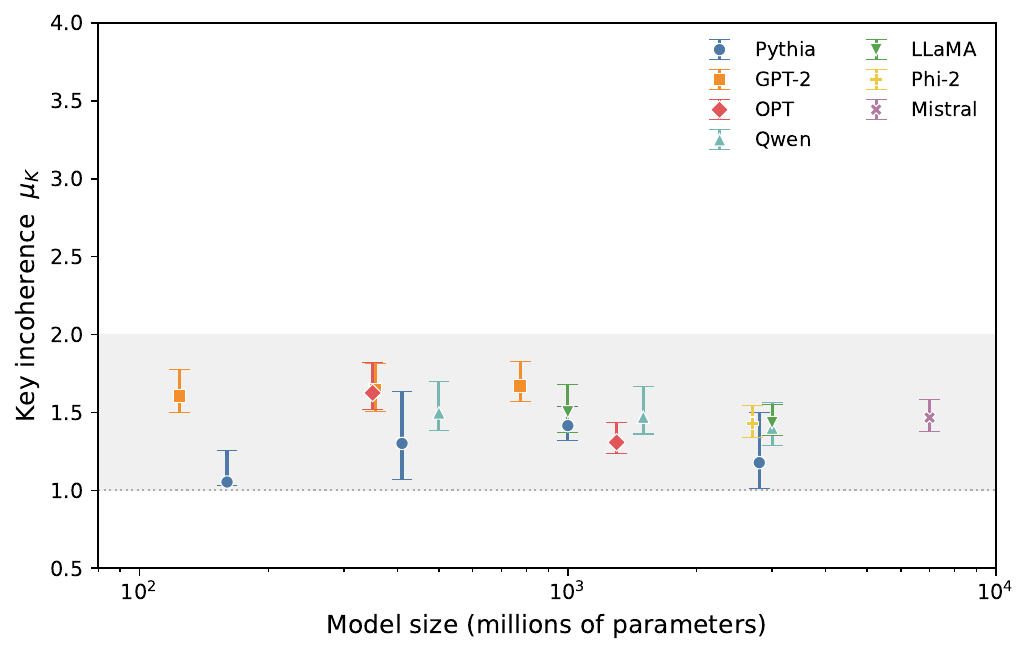}
  \caption{Key incoherence $\mu_K$ vs.\ model size for all 16 models
    at $L = 256$.  Markers show the median $\mu_K$ across all heads;
    error bars span the interquartile range.  Despite nearly two
    orders of magnitude in parameter count, $\mu_K$ shows no
    systematic trend and remains near 1.5 for every architecture
    family.  The gray dotted line marks $\mu_K = 1$, perfectly
    uniform key norms.}
  \label{fig:mu_K_vs_size}
\end{figure}

\section{Delocalization}
\label{sec:delocalization}

With $\mu_K = O(1)$ established in Section~\ref{sec:incoherence},
we can ask what it implies for the singular vectors of the energy
field.  They delocalize, spreading their mass across many positions
rather than concentrating at a few.  This section proves a bound on
the delocalization and verifies it empirically.

\begin{definition}[Inverse participation ratio]\label{def:ipr}
For a unit vector $v \in \R^L$, we define the \emph{inverse participation
ratio} as $\IPR(v) = \sum_{j=0}^{L-1} v_j^4$.  
\end{definition}

The rescaled product $\IPR(v) \cdot L$ is a dimensionless concentration 
index. It equals~$1$ for a perfectly uniform vector ($|v_j| = 1/\sqrt{L}$),
$L$ for a fully concentrated one ($v_1 = 1$ and the rest zero), and $3$
for a Gaussian random vector ($v_j \sim \mathcal{N}(0, 1/L)$).

\begin{lemma}[Delocalization bound]\label{lem:delocalization}
Let $K \in \R^{L \times d_h}$ be the key matrix of an attention
head, and let $\kappa(K) = \sigma_{\max}(K)/\sigma_{\min}(K)$
denote its condition number.  Under the assumptions
\begin{enumerate}[label=(A-\arabic*),nosep]
\item Scale: $\|K\|_F^2 / L = O(d_h)$
  \quad (automatic under layer normalization),
\item Incoherence: $\mu_K = O(1)$
  \quad (empirical; Section~\ref{sec:incoherence}),
\item Conditioning: $\kappa(K)$ bounded independently of $L$
  \quad (empirical; Remark~\ref{rem:kappa} below),
\end{enumerate}
each right singular vector $v_k$ of the row-centered logit matrix
$\tilde{E}$ satisfies
\begin{equation}\label{eq:delocalization}
\IPR(v_k) \cdot L
\;\le\; \frac{\mu_K \cdot d_h^2 \cdot \kappa^4}{L}
\;\xrightarrow{L \to \infty}\; 0.
\end{equation}
In particular, $\IPR(v_k) \cdot L = O(1)$ for any fixed $d_h$.
\end{lemma}

\begin{proof}
The row space of $Z = QK^\mathrm{T}/\sqrt{d_h}$ is a subspace of the
column space of $K$, so the right singular vectors of $Z$ lie in
$\operatorname{colspan}(K)$.  Row centering adds a rank-one
perturbation, so the right singular vectors of $\tilde{E}$ lie in
$\operatorname{colspan}(K) + \operatorname{span}(\mathbf{1})$.  We
prove the bound for any unit vector $v$ in
$\operatorname{colspan}(K)$.  The vector $\mathbf{1}/\sqrt{L}$ has $\IPR = 1/L$, so $\IPR \cdot L = 1$.
It is maximally delocalized and satisfies the bound trivially.

Write $v = Kw / \|Kw\|$ for some $w \in \R^{d_h}$.
Then $v_j = k_j^\mathrm{T} w / \|Kw\|$, so by Cauchy--Schwarz:
\[
v_j^2 \;\le\; \frac{\|k_j\|^2\,\|w\|^2}{\|Kw\|^2}.
\]
Squaring both sides: $v_j^4 \le \|k_j\|^4\,\|w\|^4 / \|Kw\|^4$.
Summing over $j$ and bounding each $\|k_j\|^2$ by
$\max_j \|k_j\|^2$:
\[
\IPR(v)
= \sum_j v_j^4
\le \frac{\max_j \|k_j\|^2 \cdot \|K\|_F^2}
         {\bigl(\|Kw\|^2/\|w\|^2\bigr)^2}.
\]
The denominator satisfies
$\|Kw\|^2/\|w\|^2 \ge \sigma_{\min}^2(K)$, so
$\IPR(v) \le \max_j\|k_j\|^2 \cdot \|K\|_F^2 / \sigma_{\min}^4(K)$.
Multiplying by $L$ and using $\max_j\|k_j\|^2 = \mu_K\|K\|_F^2/L$
(Definition~\ref{def:mu_K}):
\[
\IPR(v) \cdot L
\;\le\; \mu_K \cdot \frac{\|K\|_F^4}{L \cdot \sigma_{\min}^4(K)}.
\]
(A-1) ensures $\|K\|_F^2 = O(L\,d_h)$, keeping the bound
finite.  SVD of $K$ gives $\|K\|_F^2 = \sum_{k=1}^{d_h} \sigma_k^2(K)$.
With (A-3),
we get $\sigma_{\min}^2(K) \ge \|K\|_F^2 / (d_h\,\kappa^2)$, so
$\sigma_{\min}^4(K) \ge \|K\|_F^4 / (d_h^2\,\kappa^4)$.
Substituting:
\[
\IPR(v) \cdot L
\;\le\; \mu_K \cdot
  \frac{\|K\|_F^4}{L \cdot \|K\|_F^4 / (d_h^2\,\kappa^4)}
\;=\; \frac{\mu_K \cdot d_h^2 \cdot \kappa^4}{L}.
\qedhere
\]
\end{proof}

\begin{remark}[Conditioning assumption]
\label{rem:kappa}
(A-3) requires that $K$ uses all $d_h$ dimensions of the key
space,  so that $\sigma_{\min}(K) > 0$.  
Empirically, the median $\kappa(K)$ ranges from 20 to 570
at $L = 256$--$1{,}024$ and \emph{decreases} with context length.
For example, GPT-2 has median $\kappa = 49$ at $L = 256$ and $24$
at $L = 1{,}024$.  At short contexts ($L < d_h$), the key matrix
is rank-deficient and $\kappa$ is unbounded.  The condition number
is not $O(1)$ in the strict sense, but it does not grow with $L$,
which is what the bound requires.  A well-trained model uses all
dimensions of its key space.  Gradients repurpose any that carry
no information.
\end{remark}

Key incoherence and bounded conditioning together force
$\IPR(v_k) \cdot L \to 0$.  We verify this in 16 models at five context lengths and five
texts (Table~\ref{tab:ipr}).  In all 400 measurements,
$\IPR \cdot L$ ranges from 2.85 to 3.63, close to
the Gaussian random vector baseline of~3.  Within each model, the
cross-text CV is below 2\%.

Non-RoPE architectures (GPT-2, OPT) cluster near
$\IPR \cdot L \approx 3.4$, while RoPE-based families cluster
near 3.0.  Rotary embeddings~\cite{su2024roformer} spread key
vectors more uniformly on the $d_h$-dimensional sphere,
producing slightly more delocalized singular vectors.  The value
decreases with context length, from 3.27 at $L = 64$ to 3.04 at
$L = 1{,}024$, consistent with the bound.

\begin{table}[ht]
\centering
\caption{Delocalization ($\IPR \cdot L$) of $\tilde{E}$ at
$L = 512$.  RoPE-based architectures show systematically lower
values than non-RoPE models.}
\label{tab:ipr}
\begin{tabular}{lccc}
\\
\toprule
Family & Models & $\IPR \cdot L$ (mean) & CV (\%) \\
\midrule
GPT-2   & 3 & $3.46 \pm 0.15$ & 4.2 \\
OPT     & 2 & $3.42 \pm 0.10$ & 2.8 \\
\midrule
Pythia  & 4 & $3.05 \pm 0.08$ & 2.7 \\
Qwen    & 3 & $2.95 \pm 0.04$ & 1.3 \\
LLaMA   & 2 & $3.10 \pm 0.07$ & 2.2 \\
Phi-2   & 1 & $3.04$ & --- \\
Mistral & 1 & $3.25$ & --- \\
\midrule
\textbf{All (16 models)} & \textbf{16}
  & $\mathbf{3.17 \pm 0.21}$ & \textbf{6.7} \\
\bottomrule
\end{tabular}
\end{table}

\section{Spectral Signatures}
\label{sec:signatures}

The row-sum identity constrains the spectral structure of the
flattened signal $\mathcal{E}_t$ (Definition~\ref{def:flatten}).
Correlations between energy values at different positions and
scales are invisible in the matrix but exposed by spectral
decomposition.  The row-sum produces an exact autocovariance
bridge when squared.  The DC component also vanishes, a direct
consequence of the zero mean.  Both properties are
mechanism-level and hold for any model.  We verify them using the
discrete wavelet transform to analyze the non-stationary
flattened signal.                                                                        

\subsection{The discrete wavelet transform}\label{sec:dwt}

The flattened signal concatenates rows of increasing length and is
non-stationary by construction.  We use the discrete wavelet
transform~\cite{mallat2009wavelet} (DWT) with a Daubechies-4
wavelet in periodization mode, which handles non-stationarity and
conserves energy exactly (Parseval's theorem):
\begin{equation}\label{eq:parseval}
\sum_t \mathcal{E}_t^2
= \sum_n a_J[n]^2
  + \sum_{j=1}^{J} \sum_n d_j[n]^2,
\end{equation}
where $a_J[n]$ are approximation coefficients at the coarsest
level and $d_j[n]$ are detail coefficients at scale $2^j$.  The
\emph{approximation fraction}
$\rho = \sum_n a_J[n]^2 / \sum_t \mathcal{E}_t^2$ measures the
DC content, and the \emph{wavelet energy density}
$\mathcal{W}(j) = (1/N_j)\sum_n d_j[n]^2$ characterizes the
multi-scale correlation structure.  The flattened signal has zero
mean by construction. The wavelet approximation fraction
$\rho < 0.004$ in all 400 measurements confirms this
(Figure~\ref{fig:spectral}(a)).

\begin{figure}[!t]
  \centering
  \includegraphics[width=\textwidth]{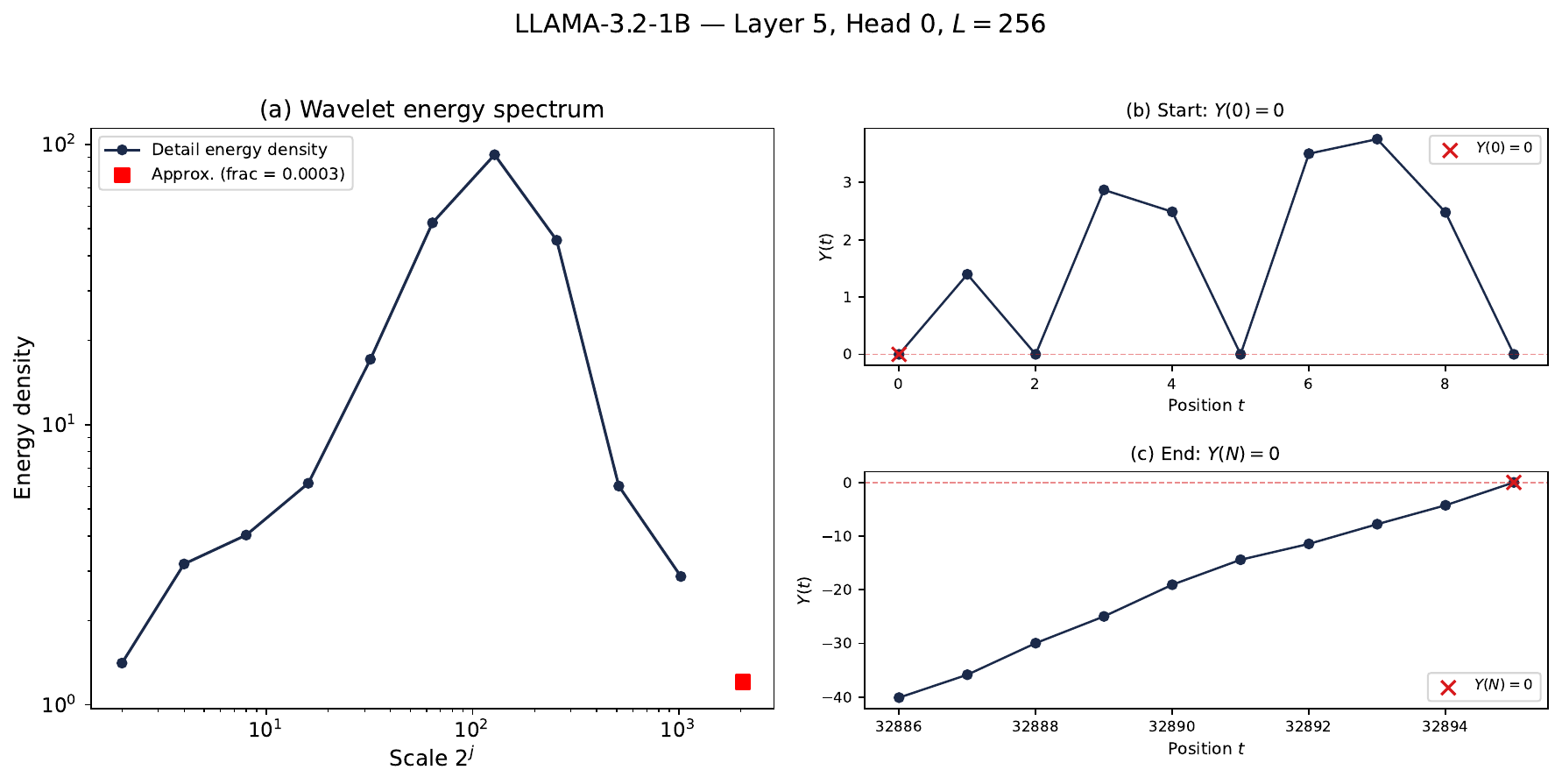}
  \caption{Spectral signatures in the flattened energy signal
    of LLaMA-3.2-1B, layer~5, head~0, $L = 256$.
    \textbf{(a)}~Wavelet energy spectrum.  Energy density increases
    with scale in the detail subbands.  The coarsest approximation
    (red square) carries $< 0.004$ of total energy, verifying the
    vanishing DC property.
    \textbf{(b, c)}~First and last 10 values of the cumulative
    signal $Y(t) = \sum_{s=1}^{t} \mathcal{E}_s$.  The signal
    starts at $Y(0) = 0$ and returns to $Y(N) = 0$, a bridge
    pinned at both ends.  This is a direct consequence of the
    row-sum identity $\sum_t \mathcal{E}_t = 0$.}
  \label{fig:spectral}
\end{figure}

\subsection{The autocovariance bridge}\label{sec:bridge}

The flattened signal pairs entries from the same row, called
\emph{within-row}, or from different rows, called
\emph{cross-row}.  For each row~$i$, the
\emph{per-row autocovariance} is
$S_i(\tau) = \sum_{j=0}^{n_i-1-\tau} E_{ij}\,E_{i,j+\tau}$.
The \emph{within-row autocovariance} is
$W(\tau) = \sum_i S_i(\tau)$, summed over rows with $n_i > \tau$,
and the \emph{cross-row autocovariance} is the remainder
$X(\tau) = N\hat{\gamma}(\tau) - W(\tau)$.

\begin{proposition}[Bridge identity]\label{prop:bridge}
The row-sum identity implies
\begin{equation}\label{eq:bridge}
\sum_{\tau \ge 0} W(\tau) = \frac{N\hat{\gamma}(0)}{2},
\qquad
\sum_{\tau=0}^{N-1} \hat{\gamma}(\tau) = \frac{\hat{\gamma}(0)}{2}.
\end{equation}
\end{proposition}

\begin{proof}
For row~$i$ with $n_i$ entries, the row-sum constraint gives
$\sum_{j=0}^{n_i-1} E_{ij} = 0$.  Squaring both sides,
\[
0 = \Bigl(\sum_{j=0}^{n_i-1} E_{ij}\Bigr)^2
  = \sum_{j}\sum_{j'} E_{ij}\,E_{ij'}.
\]
Grouping the double sum by lag $\tau = j' - j$, the $\tau = 0$
terms give $\sum_j E_{ij}^2 = S_i(0)$, and each nonzero lag
$\tau$ appears twice (once for $j' - j = \tau$, once for
$j - j' = \tau$), giving
\[
S_i(0) + 2\sum_{\tau=1}^{n_i-1} S_i(\tau) = 0.
\]
Summing over all rows~$i$,
\[
W(0) + 2\sum_{\tau \ge 1} W(\tau) = 0.
\]
Since $W(0) = \sum_i S_i(0) = \sum_t \mathcal{E}_t^2 =
N\hat{\gamma}(0)$, rearranging gives the first identity
$\sum_{\tau \ge 0} W(\tau) = N\hat{\gamma}(0)/2$.

For the second identity, apply the same argument to the global sum
$\sum_{t=1}^{N} \mathcal{E}_t = 0$.  Squaring,
\[
0 = \sum_{t=1}^{N}\sum_{t'=1}^{N}
    \mathcal{E}_t\,\mathcal{E}_{t'}.
\]
Grouping by lag $\tau = t' - t$ and using symmetry
$\hat{\gamma}(-\tau) = \hat{\gamma}(\tau)$,
\[
0 = N\hat{\gamma}(0) + 2\sum_{\tau=1}^{N-1} N\hat{\gamma}(\tau).
\]
Dividing by $N$ gives
$\hat{\gamma}(0) + 2\sum_{\tau=1}^{N-1} \hat{\gamma}(\tau) = 0$,
hence
$\sum_{\tau=1}^{N-1} \hat{\gamma}(\tau) = -\hat{\gamma}(0)/2$.
Adding $\hat{\gamma}(0)$ to include the $\tau = 0$ term,
$\sum_{\tau=0}^{N-1} \hat{\gamma}(\tau) = \hat{\gamma}(0)/2$.
\end{proof}

The cumulative signal $Y(t) = \sum_{s=1}^{t} \mathcal{E}_s$ starts
at zero and returns to zero, a \emph{bridge} pinned at both ends
(Figure~\ref{fig:spectral}(b, c)).  We verify the bridge ratio
$\sum W(\tau) / [N\hat{\gamma}(0)/2] = 1.000000$ to six decimal
places in all 400 measurements.

The bridge identity means that within-row correlations account for
exactly half the total energy, and cross-row correlations account
for the other half.  This even split holds regardless of the
model, the input, or the context length.  It is a structural
constraint on how attention organizes information, imposed by the
row-sum identity alone.

This analysis is tractable because the autocovariance decomposes
into at most $r^2 \le (d_h + 1)^2$ channel interactions
(Eq.~\eqref{eq:channel-cov}), a bound determined by head
dimension, not context length.  The wavelet transform reveals the multi-scale correlation
structure and verifies the bridge identity.

\section{Implications}
\label{sec:tools}

\subsection{Intrinsic dimensionality of the energy field}

The rank bound ($\rank(\tilde{E}) \le d_h + 1$) establishes that
the energy field lives in a low-dimensional subspace.  Trained
models use far less of this subspace than the bound allows.

The reconstruction fidelity
$F_r = 1 - \|M - M_r\|_F^2 / \|M\|_F^2$ measures the fraction of
total variance captured by a rank-$r$ approximation.
Table~\ref{tab:fidelity} compares SVD on the row-centered logit
matrix $\tilde{E}$ and the causal field $E$ against a top-$k$
sparsification baseline that keeps the $k$ largest entries per row.

\begin{table}[!h]
\centering
\caption{Reconstruction fidelity $F_r$ for three methods
(mean across GPT-2, Pythia-410M, LLaMA-3.2-1B at $L = 256$).
SVD dominates at every rank, confirming that the energy field
is low-rank, not sparse.}
\label{tab:fidelity}
\begin{tabular}{rccc}
\addlinespace[0.5em]
\toprule
$r$ & SVD($\tilde{E}$) & SVD($E$) & Top-$k$ \\
\midrule
5  & \textbf{0.89} & 0.76 & 0.25 \\
10 & \textbf{0.94} & 0.85 & 0.38 \\
20 & \textbf{0.98} & 0.92 & 0.54 \\
\bottomrule
\end{tabular}
\end{table}

At $r = 20$, SVD on $\tilde{E}$ captures 98\% of variance while
top-$k$ captures only 54\%.  The energy field is low-rank, not
sparse.  SVD on $\tilde{E}$ outperforms SVD on $E$ because the
rank bound applies to $\tilde{E}$ directly.  The causal field $E$
differs by row-dependent shifts
(Remark~\ref{rem:E-vs-Etilde}), which spread variance beyond
the rank-bounded subspace.

The intrinsic dimensionality of the energy field is $\sim$20 out
of the $d_h + 1 = 65$ dimensions the rank bound allows.  Trained
models concentrate their attention structure in far fewer
components than the mechanism permits.  
Whether this low-rank structure can be exploited at inference
time is open, because softmax amplifies reconstruction errors at
peak-attention positions.  A companion
paper~\cite{lee2026csal} proves that the attention error under
rank-$r$ truncation is bounded by the delocalization constant and
the tail singular values, providing a first bridge between the
structural findings here and inference-time fidelity.

\subsection{$\mu_K$ as a training monitor}

A head with $\mu_K \gg 1$ concentrates its key norms at one or a
few positions, collapsing attention toward fixed targets.  Monitoring $\mu_K$ during training
can detect this pathology early.

The computation is negligible.  At each logging step, for each head,
compute $\mu_K = L \cdot \max_j \|k_j\|^2 / \sum_j \|k_j\|^2$
from the key vectors.  The additional cost is one \texttt{max} and
one \texttt{sum}.

For 5{,}888 heads in 16 models, 99.8\% have
$\mu_K \le 5$ (Table~\ref{tab:mu_K}).  We suggest $\mu_K = 5$ as a
warning threshold.  The threshold is derived from post-hoc analysis
of trained models.  Validating it as a live training monitor remains
future work.

\section{Discussion}
\label{sec:discussion}

\subsection{Experimental scope and limitations}
\label{sec:methodology}

We test 16 models from seven architecture families:
Pythia~\cite{biderman2023pythia} (160M, 410M, 1B, 2.8B),
GPT-2~\cite{radford2019language} (124M, 355M, 774M),
OPT~\cite{zhang2022opt} (350M, 1.3B),
Qwen-2.5~\cite{yang2024qwen2} (0.5B, 1.5B, 3B),
LLaMA-3.2~\cite{touvron2023llama} (1B, 3B),
Phi-2~\cite{javaheripi2023phi} (2.7B), and
Mistral-7B~\cite{jiang2023mistral} (7B).
Five of the seven families use rotary position
embeddings (RoPE)~\cite{su2024roformer} and three use grouped query
attention (GQA)~\cite{ainslie2023gqa}, where fewer key--value heads
serve multiple query heads.
Models up to 3B run in float32; Mistral-7B runs in float16.

Input texts are five excerpts ($\sim$3{,}000 words each) from
Project Gutenberg: Dickens (\emph{A Tale of Two Cities}), Darwin
(\emph{On the Origin of Species}), Shakespeare (\emph{Hamlet}), the
King James Bible, and Adam Smith (\emph{The Wealth of Nations}).
The DWT pipeline is validated against synthetic ground-truth
tests, including Parseval conservation to machine precision.  All
reported quantities have CV $< 10$\% over the five texts, with
$\IPR \cdot L$ below 2\% within each model.

The scope has clear boundaries.  All models are autoregressive
language models trained on English-dominant corpora.  The largest
model (Mistral, 7B) is small by current standards; models at 70B
and above may exhibit different conditioning or attention
specialization.  Context lengths
reach $L = 1{,}024$, well below the 32K--128K windows of modern
long-context models.  The mechanism-level results (row-sum, rank bound, bridge
identity) hold by construction and are
not affected by these limitations.  The model-level results
($\mu_K = O(1)$, delocalization) are empirical claims whose
generality depends on the breadth of models tested.

\subsection{Open problems}
\label{sec:open}

The connection between $\mu_K$ and the coherence parameter in
matrix completion~\cite{candes2009exact}, noted in
Section~\ref{sec:related}, suggests that tools from random matrix
theory may apply to the energy field.

\begin{enumerate}[leftmargin=*]

\item \emph{Why does training preserve $\mu_K = O(1)$?}  Random initialization establishes incoherence.  Training develops
specialized attention patterns while maintaining it.  Deriving this preservation
from gradient descent dynamics would unify the mechanism-level and
model-level results.

\item \emph{Can the delocalization bound be tightened?}  The current
bound establishes the correct asymptotic behavior
($\IPR \cdot L \to 0$) but is quantitatively loose.  A tighter
analysis exploiting the full spectrum of $K$ would narrow the gap
between the bound and the empirical range of $2.85$--$3.63$.

\item \emph{Does $\mu_K = O(1)$ hold beyond language?}  Scaling to
70B+ models, testing at $L = 32$K--$128$K, and measuring $\mu_K$ on
vision transformers, protein models, and reinforcement learning
agents would determine whether key incoherence is specific to
natural language or general to softmax attention.

\item \emph{Can the low intrinsic dimensionality be exploited?}  The energy field concentrates most of its variance in far fewer
than $d_h + 1$ components, but the softmax nonlinearity amplifies
reconstruction errors at peak-attention positions.  Bridging this gap requires
methods that account for the interaction between low-rank
approximation and softmax.

\end{enumerate}

\section{Conclusion}
\label{sec:conclusion}

The energy field $E_{ij} = Z_{ij} - \mu_i$ is a complete,
invertible representation of softmax attention.  Two algebraic
constraints govern it.  Each row sums to zero, and the matrix has
rank at most $d_h + 1$.

The central finding is key incoherence.  Every trained language model 
we test maintains $\mu_K \approx 1.5$, keeping key norms spread over positions
despite developing specialized attention patterns.
This regularity holds in 16 models spanning seven architecture
families, from 124M to 7B parameters.  It implies that the
singular vectors of the energy field delocalize, and that the
field concentrates its variance in far fewer components than the
mechanism allows.

The energy field is more structured than the mechanism requires
and more uniform than one might expect.  Understanding why training 
preserves this uniformity remains open.

\bibliographystyle{plain}
\bibliography{oisa_ref}

\appendix
\section{Appendix: The energy field as a centered log-ratio transform}
\label{app:clr}

The arithmetic mean used to define the energy field
(Definition~\ref{def:energy}) is the unique centering that
corresponds to the geometric mean of the attention probabilities,
the natural center of a distribution on the probability simplex.

The geometric mean of the attention probabilities in row~$i$ is
$\bar{p}_i^{(G)} = \bigl(\prod_{j'} p_{ij'}\bigr)^{1/L}$.
Since $\log p_{ij} = Z_{ij} - \log \mathcal{Z}_i$, the
log-geometric-mean is
\[
\log \bar{p}_i^{(G)}
= \frac{1}{L}\sum_{j'} \log p_{ij'}
= \frac{1}{L}\sum_{j'} Z_{ij'} - \log \mathcal{Z}_i
= \bar{\mu}_i - \log \mathcal{Z}_i.
\]
The log-ratio of each probability to this geometric mean is
\[
\log \frac{p_{ij}}{\bar{p}_i^{(G)}}
= (Z_{ij} - \log \mathcal{Z}_i) - (\bar{\mu}_i - \log \mathcal{Z}_i)
= Z_{ij} - \bar{\mu}_i
= \tilde{E}_{ij}.
\]
The partition function $\mathcal{Z}_i$ cancels exactly.  The
energy field is the log-ratio of each attention probability to the
geometric mean of its row.  Positive energy means a position
receives more attention than the geometric average.  The geometric
mean, not the arithmetic mean, is the correct reference because
softmax maps logits through an exponential.

This is the \emph{centered log-ratio} (CLR) transform from
compositional data analysis~\cite{aitchison1986statistical}.
Attention probabilities are compositions.  They are nonnegative
and sum to one.  The CLR transform maps the probability simplex to
$\R^L$ with two properties.  First, the components sum to zero
(the row-sum identity of Definition~\ref{def:energy}).  Second,
the Euclidean inner product in CLR coordinates equals the
Aitchison inner product on the simplex, the natural metric for
compositional data.

By the Eckart--Young theorem, SVD minimizes the Frobenius-norm
reconstruction error at any rank.  Since the Frobenius norm in CLR
coordinates equals the Aitchison distance on the simplex, the SVD
of the energy field (Section~\ref{sec:svd}) gives the optimal
low-rank approximation of the attention pattern in the Aitchison
geometry, not merely in Euclidean distance.  The
rank bound, key incoherence, and delocalization results all hold
in the coordinate system that respects the compositional structure
of attention.

Alternative centerings fail.  Any nonlinear function of the
logits breaks the row-sum identity
$\sum_j (Z_{ij} - \bar{Z}_i) = 0$, which underpins the spectral
results.  Any centering of the probabilities other than the
geometric mean fails to cancel the partition function.  The
arithmetic mean of logits is the only centering that is both
linear in $Z$ and meaningful on the probability simplex.

The arithmetic mean of logits is the unique centering that is both
linear in $Z$ (preserving the algebraic structure) and equivalent
to the geometric mean of probabilities (respecting the
compositional geometry of the simplex).

\end{document}